\documentclass[11pt]{article}

\usepackage[margin=1in]{geometry}
\usepackage{amsmath,amssymb,amsthm}
\usepackage{graphicx}
\usepackage{booktabs}
\usepackage[round,authoryear]{natbib}
\usepackage{xcolor}
\usepackage{microtype}
\usepackage[colorlinks=true,linkcolor=blue,citecolor=blue,urlcolor=blue]{hyperref}
\graphicspath{{figures/}}
\newtheorem{proposition}{Proposition}
\hypersetup{
  pdftitle={Curvature Under Attack in hZACH-ViT: Gauge Symmetry, Boundary Saturation, and Adversarial Failure},
  pdfauthor={Athanasios Angelakis; Marta Gomez-Barrero},
  pdfsubject={Accepted at NeurReps 2026, NeurIPS 2026},
  pdfkeywords={ZACH-ViT, hZACH-ViT, geometric deep learning, hyperbolic representations, adversarial robustness, curvature, medical imaging}
}

\title{Curvature Under Attack in hZACH-ViT:\\
Gauge Symmetry, Boundary Saturation, and Adversarial Failure}

\author{
Athanasios Angelakis\textsuperscript{1,2,*}
\quad and \quad
Marta Gomez-Barrero\textsuperscript{1,\(\dagger\)}\\[0.75em]
\small \textsuperscript{1}BioML, Research Institute CODE, University of the Bundeswehr Munich, Munich, Germany\\
\small \textsuperscript{2}Amsterdam UMC, University of Amsterdam, Amsterdam, Netherlands\\[0.45em]
\footnotesize \textsuperscript{*}\href{https://orcid.org/0000-0003-1226-9560}{ORCID: 0000-0003-1226-9560}
\quad
\textsuperscript{\(\dagger\)}\href{https://orcid.org/0000-0003-4581-5353}{ORCID: 0000-0003-4581-5353}
}

\date{\small Accepted at NeurReps 2026: Symmetry and Geometry in Neural Representations, NeurIPS 2026}

\begin{document}
\maketitle

\begin{abstract}
Curvature is often treated as an intrinsic property of a representation, although its empirical effect also depends on coordinate scale, learned logit temperature, and numerical safeguards. We study this interaction in hZACH-ViT, a compact Vision Transformer with Euclidean, Poincar\'e, and spherical prototype heads. The backbone architecture, seed-specific initialization, 50-per-class training subset, and optimization protocol are matched across three MedMNIST datasets and five seeds. At the fixed comparison curvature $c=1$, Poincar\'e has the lowest class-macro PGD attack-success rate in all 12 dataset-budget cells and under a stronger CE+DLR multi-restart attack on all three datasets, but it also has the lowest clean MacroF1. An end-to-end curvature intervention changes the interpretation. Reducing Poincar\'e curvature to $c=0.1$ improves clean MacroF1 in every one of the 15 paired seed-dataset comparisons and removes hard boundary clipping, yet on OrganAMNIST it increases strong attack success from $89.7\%$ to $99.3\%$ (paired difference $+9.57$ points; 95\% hierarchical bootstrap CI $[+5.52,+14.03]$). At $c=1$, $40$-$47\%$ of clean Poincar\'e features are hard-clipped, the radial Jacobian of the inherited map is nearly zero, and dimensionless attack trajectories are unusually long and inefficient. The spherical head provides a control: its curvature change is an exact scale gauge to floating-point precision and produces much smaller attack differences. These results do not establish intrinsic hyperbolic robustness. They identify an implementation-sensitive regime in which curvature, scale, and proximity to the Poincar\'e boundary jointly organize clean recognition and adversarial representation motion.
\end{abstract}
\noindent\textbf{Keywords:}
ZACH-ViT, hZACH-ViT, geometric deep learning, hyperbolic representations,
adversarial robustness, curvature, low-data medical imaging

\section{Introduction}
Non-Euclidean representation learning replaces a usually implicit assumption - that useful features live in a flat vector space - with an explicit geometric design choice. Hyperbolic spaces provide exponentially expanding volume for hierarchical organization, while spherical representations emphasize angular structure and remove radial scale~\citep{nickel2017poincare,ganea2018hyperbolic,khrulkov2020hyperbolic,mettes2019hyperspherical}. These ideas have become established components of geometric deep learning and computer vision~\citep{bronstein2017geometric,ermolov2022hyperbolic,mettes2024survey}. Yet a trained curved model is not determined by curvature sign alone. Curvature magnitude sets a coordinate radius; a learned temperature rescales metric logits; and practical Poincar\'e implementations clamp points and denominators near the boundary. Consequently, a nominally geometric observation can reflect an interaction between intrinsic geometry, representation adaptation, and finite-precision computation.

This distinction is especially important under perturbation. Adversarial optimization provides a controlled way to interrogate which representation directions remain available and how a trajectory reaches a decision boundary. It is not, in our use, a claim about clinical security. Prior work has shown that robustness can depend on learned feature organization~\citep{ilyas2019features,tsipras2019robustness,dapello2021geometry}, while recent work has begun to formulate attacks directly for hyperbolic networks~\citep{vanspengler2024attacks}. At the same time, Poincar\'e learning is known to face precision and optimization difficulties near its boundary~\citep{sala2018tradeoffs,mishne2023stability}. High-dimensional hyperbolic prototype models can concentrate at a common near-boundary radius, sometimes erasing the practical distinction from fixed-radius Euclidean embeddings~\citep{moreira2024twosides}; clipping has therefore been used deliberately to regulate hyperbolic feature norms~\citep{guo2022clipped}. These observations motivate a more careful question than whether negative curvature is ``robust'': \emph{which part of an apparent attack slowdown is preserved when curvature scale and boundary occupancy are changed?}

We answer this question in hZACH-ViT, the curved-latent extension of the compact, permutation-invariant ZACH-ViT backbone~\citep{angelakis2026zach,angelakis2026hzach}. ZACH-ViT removes positional embeddings and the dedicated class token, uses global average pooling over patch representations, and contains approximately $0.25$M trainable parameters. hZACH-ViT retains this backbone and replaces the terminal classifier with geometry-aware prototypes. The earlier hZACH-ViT study treated geometry and curvature as dataset-dependent clean model-selection variables and found that low curvature was frequently preferred. The present work instead fixes the training and attack protocol, first comparing curvature signs at $c=1$ and then intervening within the Poincar\'e and spherical families by retraining at $c=0.1$.

The resulting contribution is fourfold. First, we provide a corrected class-macro attack comparison between linear, Euclidean-prototype, Poincar\'e, and spherical hZACH-ViT heads on identical shared-clean-correct samples. Second, we derive the curvature-scale gauge for both curved heads and use the spherical family as an exact numerical control. Third, we connect the Poincar\'e curvature intervention to normalized boundary occupancy, hard clipping, radial sensitivity, learned logit scale, and dimensionless attack trajectories. Fourth, we validate the central ordering with 20-step, three-restart CE and DLR attacks, per-sample worst-margin selection, convergence traces, and gradient diagnostics. The evidence supports an implementation-sensitive boundary regime, not a universal robustness property of negative curvature.

\section{Related Work and Positioning}
Prototypical networks classify by metric comparison with class representatives~\citep{snell2017prototypical}. Hyperbolic embeddings extend this construction to hierarchical structure~\citep{nickel2017poincare,nickel2018lorentz,ganea2018hyperbolic}, including image embeddings and Vision Transformers~\citep{khrulkov2020hyperbolic,ermolov2022hyperbolic}; hyperspherical classifiers instead remove feature magnitude and organize decisions through angles~\citep{liu2017sphereface,mettes2019hyperspherical}. The practical meaning of curvature nevertheless depends on parameterization and precision~\citep{sala2018tradeoffs,mishne2023stability}. Poincar\'e prototypes can concentrate with features at a common near-boundary radius~\citep{moreira2024twosides}, while explicit radial clipping can change classification behavior~\citep{guo2022clipped}. We therefore study curvature through normalized coordinates, hard projection, and radial sensitivity rather than treating its sign as a complete model description.

Representation geometry can constrain accessible computation in biological and artificial systems~\citep{sadtler2014constraints,chaudhuri2019attractor,gardner2022toroidal,chung2021neural}. Under perturbation, feature organization determines which non-robust directions are available~\citep{ilyas2019features,dapello2021geometry}, and direct attacks on hyperbolic networks show that evaluation itself can depend on geometry~\citep{vanspengler2024attacks}. A CVPR 2026 PHAROS workshop study evaluated common corruptions and first-order adversarial attacks for the flat ZACH-ViT backbone~\citep{angelakis2026pharos}, while the clean hZACH-ViT study examined geometry-aware heads as model-selection alternatives~\citep{angelakis2026hzach}. Neither work separated curvature sign from curvature scale, audited the curvature-scale gauge, or analyzed adversarial motion in dimensionless curved coordinates. Here attacks serve as representation probes, not evidence of deployment security.

\section{Geometry, Gauge, and Boundary Coordinates}
\subsection{Matched prototype heads}
Let $u=f_\theta(x)\in\mathbb{R}^{64}$ denote the pooled ZACH-ViT feature. A prototype head learns one preimage $p_k\in\mathbb{R}^{64}$ per class and forms
\begin{equation}
 \ell_k(u)=-s\,d_{\mathcal M,c}\!\left(\Phi_{\mathcal M,c}(u),\Phi_{\mathcal M,c}(p_k)\right),
 \label{eq:logit}
\end{equation}
where $s>0$ is learned. The Euclidean prototype head uses ordinary distance. For the Poincar\'e ball $\mathbb D_c^d=\{z:c\|z\|^2<1\}$, the inherited hZACH-ViT map is an origin exponential map followed by the canonical numerical projection into the open ball. For the sphere of radius $1/\sqrt c$,
\begin{equation}
 \Phi_{S,c}(u)=\frac{u}{\sqrt c\,\|u\|_2},\qquad
 d_{S,c}(x,y)=\frac{1}{\sqrt c}\arccos\!\left(c\langle x,y\rangle\right).
 \label{eq:sphere}
\end{equation}
All prototype variants have 248,705 active-forward parameters, compared with 248,712 for the original linear head.

\subsection{Two scale symmetries}
Positive dilation is removed exactly by spherical normalization: $\Phi_{S,c}(\alpha u)=\Phi_{S,c}(u)$ for every $\alpha>0$. For $u\ne0$, define $L(u)=\mathcal L(h(\Phi_{S,c}(u)),y)$. This gives an immediate infinitesimal constraint.
\begin{proposition}[Radial orthogonality]
For any differentiable $L$, the feature gradient satisfies
\begin{equation}
 u^\top\nabla_u L=0.
 \label{eq:orthogonal}
\end{equation}
\end{proposition}
\begin{proof}
$L(\alpha u)=L(u)$ for all $\alpha>0$. Differentiating with respect to $\alpha$ at $\alpha=1$ yields Eq.~\eqref{eq:orthogonal}.
\end{proof}

Curvature magnitude introduces a second, model-level gauge. If $c$ is changed to $c'$, spherical logits are preserved by $s'=s\sqrt{c'/c}$. Poincar\'e logits have the corresponding real-arithmetic covariance
\begin{equation}
 u'=\sqrt{c/c'}\,u,\quad p_k'=\sqrt{c/c'}\,p_k,\quad s'=s\sqrt{c'/c}.
 \label{eq:gauge}
\end{equation}
Equivalently, the normalized radius $\rho_c(z)=\sqrt c\|z\|$ and the dimensionless metric $\bar d_c=\sqrt c\,d_c$ remove the global radius scale. Equation~\eqref{eq:gauge} is a reparameterization statement, not a claim that independently trained checkpoints are equivalent. Moreover, the implemented Poincar\'e map includes float32 arithmetic, hard projection at $\rho_c=1-10^{-7}$, and stabilized denominators. Those operations can break the covariance precisely where features approach the boundary.

We therefore record four boundary-sensitive quantities: normalized radius $\rho_c$, boundary slack $1-\rho_c$, the radial Jacobian of the feature-to-ball map, and the fraction of samples reaching the hard projection threshold. For a representation trajectory $z_0,\ldots,z_T$, we additionally measure
\begin{equation}
 \bar L_c=\sum_{t=0}^{T-1}\sqrt c\,d_c(z_t,z_{t+1}),\qquad
 \bar D_c=\sqrt c\,d_c(z_0,z_T),\qquad
 \eta_c=\bar D_c/\bar L_c.
 \label{eq:efficiency}
\end{equation}
The ratio $\eta_c\in[0,1]$ is dimensionless; lower values indicate a less direct path under the same normalized metric.

\section{Experimental Protocol}
\paragraph{Datasets and low-data training.}
We use BloodMNIST (8 classes), DermaMNIST (7), and OrganAMNIST (11) from MedMNIST v2~\citep{yang2023medmnist}. Following the ZACH-ViT protocol, each seed uses 50 training images per class, 224$\times$224 RGB input, patch size 16, batch size 16, Adam at $10^{-4}$ without weight decay, and 47 epochs. We use seeds $\{3,5,7,11,13\}$ and select by validation MacroF1 (validation loss breaks ties) before accessing the official test set. The $c=1$ study reuses 60 selected checkpoints. We train 30 missing end-to-end $c=0.1$ Poincar\'e and spherical models with exactly matched training indices and order.

\paragraph{Shared-clean-correct evaluation.}
ASR is conditioned on clean correctness. The $c=1$ pool intersects images correctly classified by all four heads; curvature comparisons use a separate within-family intersection for $(c=1,c=0.1)$. Core and strong subsets retain up to 64 and 16 images per eligible class. We compute class-macro ASR within seed and then average the five seeds:

For seed $r$, let $C_r$ be the classes represented in the relevant shared pool and $I_{r,k}$ the selected clean-correct images of class $k$. We report
\begin{equation}
 \widehat{\mathrm{ASR}}_r=\frac{1}{|C_r|}\sum_{k\in C_r}\frac{1}{|I_{r,k}|}
 \sum_{i\in I_{r,k}}\mathbf 1\!\left[h(x_i^{\mathrm{adv}})\ne k\right],
 \qquad \widehat{\mathrm{ASR}}=\frac{1}{5}\sum_r\widehat{\mathrm{ASR}}_r.
 \label{eq:asr}
\end{equation}
Curvature effects are paired differences on the same images. Intervals use 3,000 hierarchical percentile-bootstrap draws that resample seeds and then images within class. Clean MacroF1 uses the full test set; coverage is detailed in Appendix~\ref{app:coverage}.

\paragraph{Attacks and diagnostics.}
Core attacks use untargeted, clean-start PGD-10 with cross-entropy, $\epsilon\in\{1,2,4,8\}/255$, and step size $\epsilon/4$~\citep{goodfellow2015explaining,madry2018towards}. Strong validation at $4/255$ uses PGD-20 with three random starts each for CE and DLR; the smallest true-class margin among all six candidates is retained. We audit convergence, input gradients, and selected losses. Trajectories use up to eight images per class and Eq.~\eqref{eq:efficiency}. These are empirical probes, not certificates~\citep{athalye2018obfuscated,croce2020reliable}.

\section{Results}
\subsection{At c=1, Poincar\'e fails more slowly but recognizes less well}
Figure~\ref{fig:c1} and Table~\ref{tab:c1} give the corrected class-macro comparison. Poincar\'e has the lowest point-estimate PGD-10 ASR in all 12 dataset-budget cells. The separation is largest on OrganAMNIST: at $4/255$, Poincar\'e is $8.05$ points below the spherical head, $8.52$ below the Euclidean prototype, and $9.37$ below the linear head; all three paired hierarchical intervals exclude zero. Under the stronger attack, Poincar\'e again has the lowest point estimate on every dataset. The OrganAMNIST differences remain $-9.38$ points relative to sphere (95\% CI $[-14.66,-4.59]$), $-8.78$ relative to Euclidean prototype ($[-15.07,-2.36]$), and $-9.97$ relative to linear ($[-15.19,-5.06]$). BloodMNIST and DermaMNIST show smaller or interval-uncertain pairwise differences.

\begin{figure}[htbp]
\centering
\includegraphics[width=\linewidth]{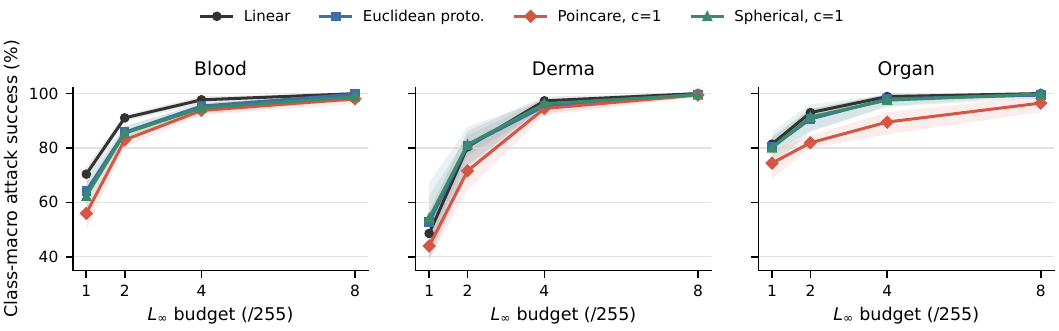}
\caption{PGD-10 class-macro ASR for the fixed $c=1$ comparison. Curves are five-seed means and shaded regions are hierarchical 95\% bootstrap intervals. Every head is attacked on the identical four-head shared-clean-correct subset within each seed. Lower ASR means slower empirical failure under this attack, not certified robustness.}
\label{fig:c1}
\end{figure}

\begin{table}[htbp]
\centering
\caption{Clean MacroF1 and ASR (\%) at $4/255$, averaged over five seeds. PGD-10 and strong attacks use independently capped versions of the same four-head shared-clean-correct pool. Bold marks the clean-best or attack-slowest head within dataset.}
\label{tab:c1}
\scriptsize
\begin{tabular}{llrrr}
\toprule
Dataset & Head & Clean F1 & PGD-10 & PGD-20/R3 CE+DLR\\
\midrule
Blood & Linear & \textbf{.805} & 97.70 & 96.88\\
& Euclidean prototype & .783 & 95.39 & 96.41\\
& Poincar\'e, $c=1$ & .714 & \textbf{93.79} & \textbf{93.59}\\
& Spherical, $c=1$ & .770 & 94.73 & 94.84\\
\midrule
Derma & Linear & \textbf{.354} & 97.31 & 98.33\\
& Euclidean prototype & .340 & 95.36 & 97.71\\
& Poincar\'e, $c=1$ & .317 & \textbf{94.50} & \textbf{96.25}\\
& Spherical, $c=1$ & .338 & 96.27 & 98.33\\
\midrule
Organ & Linear & .506 & 98.88 & 99.32\\
& Euclidean prototype & \textbf{.518} & 98.03 & 98.13\\
& Poincar\'e, $c=1$ & .405 & \textbf{89.51} & \textbf{89.35}\\
& Spherical, $c=1$ & .489 & 97.57 & 98.72\\
\bottomrule
\end{tabular}
\end{table}

This ordering is accompanied by a clean-performance cost. Poincar\'e has the lowest clean MacroF1 on all three datasets. The observation is therefore not a free improvement in robustness: at $c=1$, the head that is hardest for these attacks is also the least useful clean classifier.

\subsection{The curvature intervention removes the clean penalty and the Organ attack slowdown}
Retraining at $c=0.1$ produces a consistent Poincar\'e clean improvement (Fig.~\ref{fig:mechanism}a). MacroF1 increases by $+7.69$, $+3.34$, and $+9.95$ percentage points on Blood, Derma, and OrganAMNIST, respectively; all 15 within-seed differences are positive. The improvement is not achieved by preserving the $c=1$ attack ordering. On OrganAMNIST, strong ASR rises from $89.70\%$ to $99.27\%$ (difference $+9.57$ points, 95\% CI $[+5.52,+14.03]$), with the same sign in all five seeds. PGD-10 shows the same direction at every budget: $+5.58$, $+8.99$, $+9.94$, and $+5.02$ points from $1/255$ through $8/255$, each with a positive hierarchical interval. On BloodMNIST and DermaMNIST, where strong attacks are already near saturation, the corresponding strong differences are $+0.94$ and $+0.63$ points and their intervals include zero.

\begin{figure}[!t]
\centering
\includegraphics[width=.92\linewidth]{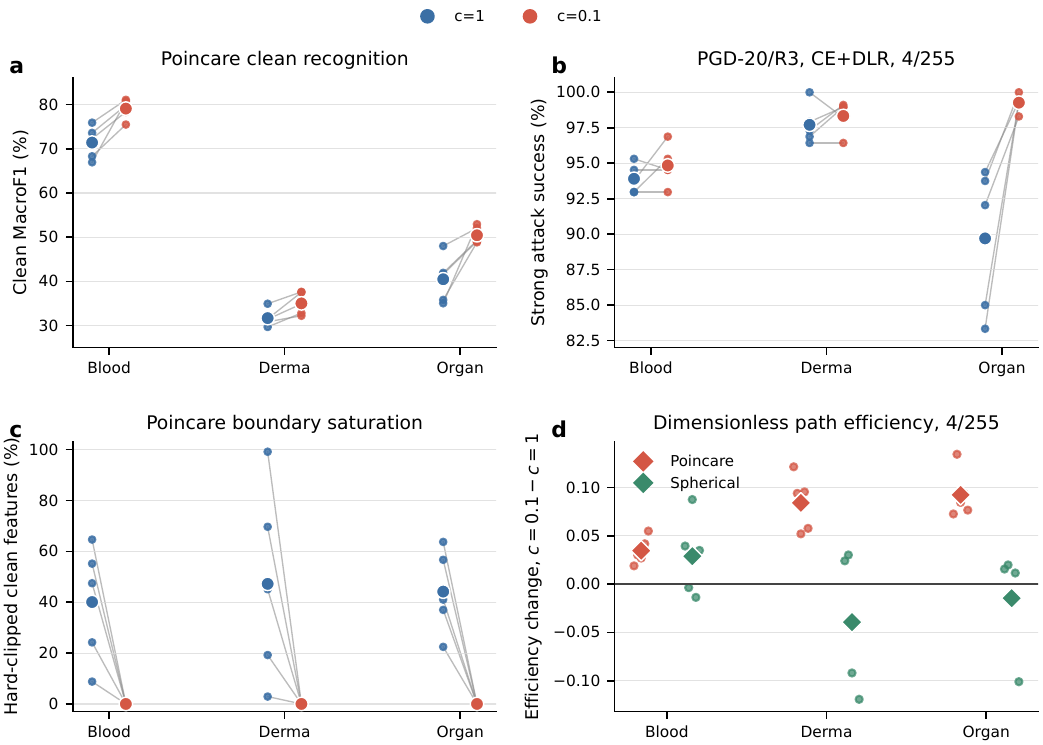}
\caption{End-to-end curvature intervention. Small points are seeds, large points/diamonds are five-seed means, and lines in (a-c) connect matched seeds. (a) Poincar\'e clean MacroF1. (b) Strong class-macro ASR. (c) Fraction of clean Poincar\'e features reaching the hard projection threshold. (d) Change in dimensionless PGD-10 path efficiency at $4/255$; positive values indicate a more direct trajectory at $c=0.1$. The spherical family is the curvature-gauge control.}
\label{fig:mechanism}
\end{figure}

\begin{table}[htbp]
\centering
\caption{Poincar\'e end-to-end curvature sensitivity. Radius and clipping entries are averages of seed-level clean diagnostics. Arrows run from $c=1$ to $c=0.1$.}
\label{tab:curvature}
\scriptsize
\begin{tabular}{lrrrr}
\toprule
Dataset & Clean MacroF1 & Strong ASR (\%) & Median $\rho_c$ & Hard clip (\%)\\
\midrule
Blood & .714$\rightarrow$.791 & 93.91$\rightarrow$94.84 & 1.000$\rightarrow$.981 & 40.1$\rightarrow$0\\
Derma & .317$\rightarrow$.350 & 97.71$\rightarrow$98.33 & 1.000$\rightarrow$.972 & 47.2$\rightarrow$0\\
Organ & .405$\rightarrow$.504 & 89.70$\rightarrow$99.27 & 1.000$\rightarrow$.943 & 44.2$\rightarrow$0\\
\bottomrule
\end{tabular}
\end{table}

The spherical control changes much less under the same intervention. Its strong ASR differences are $+0.16$, $+0.33$, and $+1.48$ points on Blood, Derma, and OrganAMNIST. Only the Organ interval is positive ($[+0.11,+3.07]$), and the magnitude is substantially smaller than for Poincar\'e. This family-specific contrast argues against explaining the result by the nominal manifold radius alone.

\subsection{Boundary occupancy, radial sensitivity, and the curvature gauge}
At $c=1$, mean seed-level median Poincar\'e radius rounds to $1.000$ on every dataset, and $40.1$-$47.2\%$ of clean features reach the hard projection threshold (Table~\ref{tab:curvature}). At $c=0.1$, clipping disappears in all 15 runs and representations move inward. Mean median radial Jacobians increase from $1.0\times10^{-5}$, $1.8\times10^{-5}$, and $3.3\times10^{-5}$ to $3.7\times10^{-2}$, $5.6\times10^{-2}$, and $1.11\times10^{-1}$ on Blood, Derma, and Organ, respectively. The intervention therefore changes access to radial representation motion as well as clean fit.

Learned temperature does not compensate: Poincar\'e's mean active scale is $1.139$ at both curvatures, so $s/\sqrt c$ changes from $1.139$ to $3.603$. In a functional gauge audit, spherical rescaling changes logits by at most $1.43\times10^{-6}$ across 133,242 logits and never changes a prediction. The boundary-saturated Poincar\'e $1\rightarrow0.1$ transform instead reaches a maximum difference of $0.713$ and changes one prediction among 62,572 logits; the reverse transform changes none. Thus Eq.~\eqref{eq:gauge} is exact in ideal arithmetic but not at the inherited Poincar\'e boundary. Full directional errors appear in Appendix~\ref{app:gauge}.

\subsection{Dimensionless trajectories and attack validity}
Normalizing metric radius does not remove the trajectory signature. At $4/255$, Poincar\'e $c=1$ efficiencies are $0.120$, $0.123$, and $0.118$ on Blood, Derma, and Organ, versus $0.155$, $0.207$, and $0.211$ at $c=0.1$. Dimensionless $c=1$ paths are three to five times longer, and $40.4$-$57.9\%$ of their states are clipped; no $c=0.1$ state is clipped. Spherical changes are smaller and not directionally uniform (Fig.~\ref{fig:mechanism}d; Appendix~\ref{app:trajectory}).

All 60 curvature-comparison combinations have finite, non-zero clean input gradients. On OrganAMNIST, CE and DLR strengthen through 20 steps but retain the same curvature separation; DLR also contributes selected worst-margin candidates. These checks narrow the interpretation from generic gradient failure to the measured boundary regime, without making PGD a certificate (Appendix~\ref{app:attackvalidity}). A frozen-backbone control further shows that a Poincar\'e terminal distance alone does not reproduce the end-to-end ordering (Appendix~\ref{app:frozen}); representation adaptation is part of the effect.

\subsection{Frozen-backbone attribution control}
To separate terminal distance from learned representation adaptation, we freeze each selected linear BloodMNIST backbone and train only Euclidean, Poincar\'e, and spherical prototype heads. On the strict direct four-head pool, five-seed PGD-10 ASR at $4/255$ is $98.0\%$ (linear), $95.0\%$ (Euclidean prototype), $97.5\%$ (Poincar\'e), and $96.9\%$ (sphere). Poincar\'e is therefore not the slowest head on this common representation. Because several seed pools contain only 3-40 attacked images, this is an attribution control rather than a robustness ranking. Together with the end-to-end intervention, it rejects the strongest head-only interpretation: the $c=1$ ordering depends on interaction among the training objective, the adapted backbone, and the boundary-sensitive map (Appendix~\ref{app:frozen}).

\section{Discussion}
The $c=1$ result admits a tempting narrative - Poincar\'e representations fail more slowly - but the intervention makes it incomplete. Lower curvature improves every paired Poincar\'e clean result, removes saturation, restores radial sensitivity, and removes the large Organ attack gap. The same change is an almost exact numerical gauge for the sphere and has a much smaller adversarial effect. This connects near-boundary concentration in prototype learning~\citep{moreira2024twosides}, precision-dependent hyperbolic optimization~\citep{sala2018tradeoffs,mishne2023stability}, and the effects of radial clipping~\citep{guo2022clipped}: boundary occupancy here coincides with extinguished radial sensitivity and inefficient normalized attack paths.

The symmetry distinction is precise. Spherical normalization quotients every positive radial orbit, so the feature-loss gradient is tangent and attacks act by rotating direction. Poincar\'e geometry retains radius intrinsically, yet the implemented $c=1$ map often suppresses radial sensitivity at the numerical boundary. One is an exact architectural invariance; the other is a curvature-dependent finite-precision regime. Because $c=1$ and $c=0.1$ models are trained independently, the study establishes their coordinated changes but does not prove that clipping alone causes the ASR gap.

\section{Limitations and Broader Impact}
We study one compact architecture, three medical benchmarks, five seeds, and 50 examples per class. Results may change with full-data training, larger backbones, other Poincar\'e maps, or adversarial training. Shared-clean-correct evaluation excludes classes without jointly correct images, and the stabilized trajectory calculation need not duplicate the classifier's clamp sequence bit-for-bit. The attacks are strong first-order tests, not exhaustive or certified evaluations; ASR is high for every head at $4/255$. No configuration is clinically secure. The diagnostic lesson is that curvature and boundary numerics must be audited together before attributing failure to intrinsic geometry.

\section{Conclusion}
At $c=1$, Poincar\'e hZACH-ViT fails more slowly in every core cell, but pays a clean-recognition cost. At $c=0.1$, that cost, hard clipping, and the large Organ attack slowdown disappear, whereas the spherical gauge control changes much less. The supported claim is therefore narrower than ``hyperbolic geometry is robust'': \emph{curvature can place a learned Poincar\'e representation in a boundary-saturated numerical regime that jointly reorganizes clean recognition and adversarial motion}.

\section*{Reproducibility}
The experiments were conducted with Python 3.10.16, PyTorch 2.3.1+cu121, CUDA 12.1, and one NVIDIA GeForce RTX 3060 GPU. The code and supporting materials required to reproduce the reported analyses will be released publicly after NeurIPS 2026.

\clearpage
\bibliography{references}

\appendix
\section{Evaluation-pool coverage}\label{app:coverage}
The four-head $c=1$ core pools contain 454-512 images per BloodMNIST seed, 176-205 per DermaMNIST seed, and 423-704 per OrganAMNIST seed after the independent class cap. All BloodMNIST pools contain all eight classes. DermaMNIST consistently lacks class 3 in the strict four-head intersection. OrganAMNIST covers 8-11 classes depending on seed. The strong pools use an independent cap of 16 images per eligible class.

The within-family curvature pools are larger because they require agreement only between $c=1$ and $c=0.1$. Poincar\'e core pools contain 502-512 BloodMNIST images, 207-275 DermaMNIST images, and 554-703 OrganAMNIST images; their strong counterparts contain 128, 89-103, and 144-176 images, respectively. The spherical curvature pools contain all 11 OrganAMNIST classes in every seed. Exact per-class counts and selected indices are included in the supplementary package.

\section{Functional curvature-scale gauge audit}\label{app:gauge}
Features, prototypes, and temperature are transformed together according to Eq.~\eqref{eq:gauge}; logits and predictions are then recomputed through the implemented maps on the full clean test set. The directional asymmetry in Table~\ref{tab:gauge} concentrates in boundary-saturated $c=1$ Poincar\'e runs and is absent from the spherical control.

\begin{table}[htbp]
\centering
\caption{Curvature-scale gauge audit. Mean errors are weighted by the number of logits.}
\label{tab:gauge}
\scriptsize
\begin{tabular}{llrrrr}
\toprule
Family & Direction & Max $|\Delta\ell|$ & Mean $|\Delta\ell|$ & Pred. changes & Logits\\
\midrule
Sphere & $1\rightarrow0.1$ & $4.77\times10^{-7}$ & $7.27\times10^{-8}$ & 0 & 66,621\\
Sphere & $0.1\rightarrow1$ & $1.43\times10^{-6}$ & $2.06\times10^{-7}$ & 0 & 66,621\\
Poincar\'e & $1\rightarrow0.1$ & $7.13\times10^{-1}$ & $7.46\times10^{-3}$ & 1 & 62,572\\
Poincar\'e & $0.1\rightarrow1$ & $3.99\times10^{-1}$ & $5.52\times10^{-4}$ & 0 & 62,572\\
\bottomrule
\end{tabular}
\end{table}

\section{Strong-attack convergence and gradients}\label{app:attackvalidity}
Across the curvature comparison, the maximum observed zero-gradient fraction and non-finite-gradient fraction are both zero. Median clean input-gradient $L_2$ norms are larger at $c=0.1$: for Poincar\'e they average $0.272$ versus $0.234$ on Blood, $0.551$ versus $0.254$ on Derma, and $3.72$ versus $2.55$ on OrganAMNIST. DLR supplies the worst-margin candidate for a non-zero fraction of images in every dataset family, so the combined result is not a relabeling of CE-only PGD. The strong-attack convergence curves are presented in Fig.~\ref{fig:convergence}; the supplement contains every seed-level convergence record and selected loss family.

\begin{figure}[htbp]
\centering
\includegraphics[width=.70\linewidth]{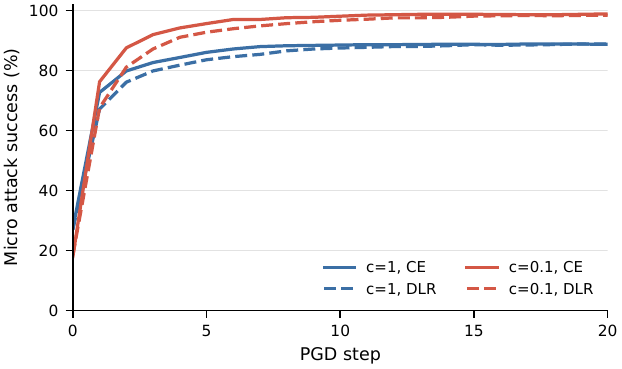}
\caption{OrganAMNIST Poincar\'e convergence under the first restart of the strong attack, averaged over five seeds. CE and DLR strengthen through 20 steps while preserving the $c=1$ versus $c=0.1$ separation. Final ASR uses the worst true-class margin over all six CE/DLR restart candidates.}
\label{fig:convergence}
\end{figure}

\section{Frozen-backbone attribution control}\label{app:frozen}
For each BloodMNIST seed, we froze the validation-selected linear ZACH-ViT feature extractor and trained only Euclidean, Poincar\'e, and spherical prototype heads. On the strict direct four-head shared-clean-correct pool, the five-seed mean PGD-10 ASRs at $4/255$ are 98.0\% (linear), 95.0\% (Euclidean prototype), 97.5\% (Poincar\'e), and 96.9\% (sphere). The pool is small in several seeds (3-40 attacked images), so this is an attribution control rather than a robustness ranking. It shows that a Poincar\'e terminal metric acting on the representation learned under the linear objective is insufficient to reproduce the end-to-end $c=1$ ordering.

In the same frozen representation, changing curvature from $c=1$ to $c=0.1$ changes Poincar\'e PGD-10 ASR by $-1.9$, $+0.9$, $0.0$, and $+0.9$ points across the four budgets; spherical differences are $+1.0$, $-1.8$, $-0.01$, and $-0.09$ points. The corresponding clean MacroF1 nevertheless increases from .433 to .557 for Poincar\'e and from .542 to .591 for sphere. Clean curvature selection and perturbation dynamics therefore do not share a fixed ordering on this common frozen representation.

\section{Dimensionless trajectory summary}\label{app:trajectory}
\begin{table}[htbp]
\centering
\caption{Five-seed mean PGD-10 trajectory quantities at $4/255$. Path and endpoint use $\sqrt c d_c$. ``Clip steps'' is the fraction of trajectory states at the Poincar\'e hard projection threshold.}
\label{tab:trajectory}
\scriptsize
\begin{tabular}{llrrrr}
\toprule
Dataset & Configuration & Endpoint & Path & Efficiency & Clip steps\\
\midrule
Blood & Poincar\'e $c=1$ & 24.33 & 205.70 & .120 & .404\\
      & Poincar\'e $c=0.1$ & 8.67 & 57.82 & .155 & .000\\
Derma & Poincar\'e $c=1$ & 24.14 & 199.55 & .123 & .449\\
      & Poincar\'e $c=0.1$ & 7.71 & 41.31 & .207 & .000\\
Organ & Poincar\'e $c=1$ & 23.97 & 215.27 & .118 & .579\\
      & Poincar\'e $c=0.1$ & 7.38 & 58.24 & .211 & .000\\
\midrule
Blood & Sphere $c=1$ & 1.38 & 4.20 & .431 & -\\
      & Sphere $c=0.1$ & 1.36 & 3.61 & .460 & -\\
Derma & Sphere $c=1$ & 1.17 & 2.85 & .569 & -\\
      & Sphere $c=0.1$ & 1.23 & 3.00 & .530 & -\\
Organ & Sphere $c=1$ & 1.25 & 5.40 & .429 & -\\
      & Sphere $c=0.1$ & 1.22 & 4.98 & .414 & -\\
\bottomrule
\end{tabular}
\end{table}

\end{document}